\documentclass{article}

\usepackage{microtype}
\usepackage{graphicx}
\usepackage{booktabs} % for professional tables

\usepackage{hyperref}

\usepackage[accepted]{icml2021}

\usepackage{etoolbox}
\usepackage{latexsym}
\usepackage{amsmath}
\usepackage{amssymb}
\usepackage{amsthm}
\usepackage{nccmath}
\usepackage{xspace}
\usepackage[noabbrev,capitalize,nameinlink]{cleveref}
\usepackage{multirow}
\usepackage{float}
\usepackage{caption}
\usepackage{subcaption}
\usepackage{comment}
\usepackage{enumitem}
\usepackage{titlesec}
\usepackage{placeins}
\usepackage{dblfloatfix}
\usepackage{bbding}

\makeatletter
\newtheorem*{rep@theorem}{\rep@title}
\newcommand{\newreptheorem}[2]{%
\newenvironment{rep#1}[1]{%
 \def\rep@title{#2 \ref{##1}}%
 \begin{rep@theorem}}%
 {\end{rep@theorem}}}
\makeatother

\newtheorem{theorem}{Theorem}
\newtheorem{proposition}{Proposition}

\newtheorem{remark}{Remark}

\newtheorem{example}{Example}

\newreptheorem{theorem}{Theorem}
\newreptheorem{proposition}{Proposition}
\newreptheorem{lemma}{Lemma}

\newcommand{\fml}[1]{{\mathcal{#1}}}

\newcommand{\tn}[1]{\textnormal{#1}}

\newcommand{\axp}{\ensuremath\mathsf{findAXp}}
\newcommand{\cxp}{\ensuremath\mathsf{findCXp}}

\newcommand{\mkfix}{\ensuremath\mathsf{FixAttr}}
\newcommand{\mkuniv}{\ensuremath\mathsf{FreeAttr}}
\newcommand{\SAT}{\ensuremath\mathsf{SAT}}
\newcommand{\outc}{\ensuremath\mathsf{outc}}
\newcommand{\prtaxp}{\ensuremath\mathsf{reportAXp}}
\newcommand{\prtcxp}{\ensuremath\mathsf{reportCXp}}

\newcommand{\mbf}[1]{\ensuremath\mathbf{#1}}
\newcommand{\mbb}[1]{\ensuremath\mathbb{#1}}
\newcommand{\msf}[1]{\ensuremath\mathsf{#1}}

\newcommand{\keepmark}{{\small\Checkmark}}
\newcommand{\dropmark}{{\small\XSolidBrush}}

\DeclareMathOperator*{\limply}{\rightarrow}

\AtBeginDocument{
}

\makeatletter
\newcommand{\ALC@uniqueautorefname}{line}
\makeatother

\setlist{nolistsep}

\expandafter\def\expandafter\normalsize\expandafter{%
    \normalsize
    \setlength\abovedisplayskip{2.0pt}
    \setlength\belowdisplayskip{2.0pt}
    \setlength\abovedisplayshortskip{2.0pt}
    \setlength\belowdisplayshortskip{2.0pt}
}

\titlespacing{\section}{0pt}{*1.75}{*1.0}
\titlespacing{\subsection}{0pt}{*1.0}{*0.75}
\titlespacing{\subsubsection}{0pt}{*1.0}{*0.25}
\titlespacing{\paragraph}{0pt}{*0.25}{1em}

\setlist{nosep}

\begin{document}

% Paper can use 8 pages

\icmltitlerunning{Explanations for Monotonic Classifiers
}

\twocolumn[
  \icmltitle{Explanations for Monotonic Classifiers}
  


\icmlsetsymbol{equal}{*}

\begin{icmlauthorlist}
\icmlauthor{Joao Marques-Silva}{irit}
\icmlauthor{Thomas Gerspacher}{irit}
\icmlauthor{Martin Cooper}{irit}
\icmlauthor{Alexey Ignatiev}{mon}
\icmlauthor{Nina Narodytska}{vmr}
%\icmlauthor{Tateu H.~Yasehe}{ed,to,goo}
%\icmlauthor{Aaoeu Iasoh}{goo}
%\icmlauthor{Buiui Eueu}{ed}
%\icmlauthor{Aeuia Zzzz}{ed}
%\icmlauthor{Bieea C.~Yyyy}{to,goo}
%\icmlauthor{Teoau Xxxx}{ed}
%\icmlauthor{Eee Pppp}{ed}
\end{icmlauthorlist}

\icmlaffiliation{irit}{IRIT, CNRS, Universit\'{e} Paul Sabatier, Toulouse, France}
\icmlaffiliation{mon}{Monash University, Melbourne, Australia}
\icmlaffiliation{vmr}{VMware Research, CA, USA}
\icmlcorrespondingauthor{Joao Marques-Silva}{joao.marques-silva@irit.fr}
%\icmlcorrespondingauthor{Eee Pppp}{ep@eden.co.uk}

\icmlkeywords{Machine Learning, ICML}

\vskip 0.3in
]



\printAffiliationsAndNotice{\icmlEqualContribution} % otherwise use the standard text.

\begin{abstract}
  In many classification tasks there is a requirement of monotonicity.
  Concretely, if all else remains constant, increasing
  (resp.~decreasing) the value of one or more features must not
  decrease (resp.~increase) the value of the prediction.
  Despite comprehensive efforts on learning monotonic classifiers,
  dedicated approaches for explaining monotonic classifiers are
  scarce and classifier-specific.
  This paper describes novel algorithms for the computation of one
  formal explanation of a (black-box) monotonic classifier. These
  novel algorithms are polynomial %(indeed linear) 
  in the run time complexity of the classifier %.
  and the number of features.
  Furthermore, the paper presents a practically efficient
  model-agnostic algorithm for enumerating formal explanations. 
  %
  %
  %This document provides a basic paper template and submission
  %guidelines. Abstracts must be a single paragraph, ideally between
  %4--6 sentences long. Gross violations will trigger corrections at
  %the camera-ready phase.
\end{abstract}

\section{Introduction} \label{sec:intro}

Monotonicity is an often required constraint in practical applications
of machine learning.
%~\cite{ben-david-ci89,ben-david-ml95,sill-nips97,bioch-amai98,bioch-idaj00,sill-ml08,gupta-jmlr16,gupta-nips16,gupta-nips17,lafferty-icml18,reynolds-aaai19,louppe-nips19,tzamos-icml20,gupta-aistats20,vandenbroeck-nips20,liu-nips20},
%with extensive work in classification problems~\cite{garcia-neuroc19}.
%
Broadly, a monotonicity constraint requires that increasing
(resp.~decreasing) the value of one or more features, while keep the
other features constant, will not cause the prediction to decrease
(resp.~increase).
%
%if all else remains constant, increasing (resp.~decreasing) the value
%of one or more features must not decrease (resp.~increase) the value
%of the prediction.
%
Monotonicity has been investigated in the context of
classification~\cite{garcia-neuroc19}, including
neural
networks~\cite{sill-nips97,sill-ml08,lafferty-icml18,vandenbroeck-nips20,liu-nips20},
random forests~\cite{reynolds-aaai19} and rule
ensembles~\cite{reynolds-icde18},
decision trees~\cite{ben-david-ci89,ben-david-ml95}, decision
lists~\cite{bioch-idaj00} and decision rules~\cite{baesens-asc17},
support vector machines~\cite{reynolds-adma16},
nearest-neighbor classifiers~\cite{feelders-ecml08},
among others~\cite{gupta-nips16,gupta-jmlr16,gupta-nips17,lafferty-icml18}.
Monotonicity has been studied in 
bayesian networks~\cite{feelders-uai04,darwiche-ijcai18},
active learning~\cite{feelders-sdm12} and, more recently, in
fairness~\cite{gupta-aistats20}.

To a much lesser extent, monotonicity has also been studied from the
perspective of explainability, with one recent example being the 
study of the explainability of monotonic bayesian
networks~\cite{darwiche-ijcai18}. This work proposes to compile
different families of bayesian networks, including naive bayes and
monotonic networks, into a decision diagram, which can then be used
for computing PI-explanations\footnote{% 
  Given some feature space point $\mbf{v}$, a PI-explanation is a
  subset-minimal subset of features which, the assignment of the 
  corresponding coordinate value in $\mbf{v}$, is sufficient for the
  prediction.}.
%Compilation approaches are worst-case exponential both in time and in
%the size of the model representation, due to the compilation step,
%with explanations efficiently computed on the compiled
%representation.
%
Approaches based on an intermediate (knowledge) compilation step are
characterized by two main drawbacks, namely their worst-case
complexity, which is exponential both in time and in the size of the
representation, but also the fact that these approaches are not
model-agnostic, i.e.\ some formal logic representation of the model
must be known and reasoned about.
Clearly, model-agnostic heuristic approaches, which include
LIME~\cite{guestrin-kdd16}, SHAP~\cite{lundberg-nips17}, or
Anchor~\cite{lundberg-nips17}, can also be applied to explaining
monotonic classifiers. However, these approaches do not readily
exploit monotonicity, and both the theoretical and practical
performance may be discouraging\footnote{%
  In fact, there are recent negative results on the tractability of
  exact SHAP learning~\cite{vandenbroeck-corr20b}.}. Furthermore,
heuristic approaches offer no formal guarantees of rigor, e.g.\ an
Anchor explanation may be consistent with points in feature space
for which the model's prediction differ from the target
prediction~\cite{ignatiev-ijcai20}.

On a more positive note, recent work proposed polynomial-time exact
algorithms for computing PI-explanations explanations of different
classes of classifiers~\cite{msgcin-nips20}, namely linear and naive
bayes classifiers. These results were complemented by the observation
that, for ML models related with some classes of %propositional
knowledge representation languages, PI-explanations can also be
computed in polynomial time~\cite{marquis-kr20}.

This paper extends these initial results to the case of monotonic
classifiers, in a number of ways.
First, the paper proposes model-agnostic algorithms for computing
PI-explanations and contrastive explanations~\cite{miller-aij19} for
\emph{any} monotonic ML model.
Second, the complexity of the proposed algorithms is shown to be
polynomial %, in fact linear,
on the time required to run the (black-box) monotonic classifier
and the number of features.
Third, the paper proposes an algorithm for the iterative enumeration
of formal explanations\footnote{%
  The term formal explanation is used in contrast with heuristic
  explanation~\cite{guestrin-kdd16,lundberg-nips17,guestrin-aaai18}
  and it will be defined precisely in~\autoref{sec:prelim}.}.
(This algorithm is worst-case exponential, but it is shown to be
remarkably efficient in practice.)

The paper is organized as follows.
\autoref{sec:prelim} introduces the notation and definitions used in
the rest of the paper.
\autoref{sec:xps} details algorithms for computing one or more formal
explanations of monotonic classifiers.
\autoref{sec:res} summarizes initial experiments, which confirm the
scalability of the proposed algorithms.
The paper concludes in~\autoref{sec:conc}.

\section{Preliminaries} \label{sec:prelim}

\paragraph{Classification problems.}
A classification problem is defined on a set of features (or
attributes) $\fml{F}=\{1,\ldots,N\}$ and a set of classes
$\fml{K}=\{c_1,c_2,\ldots,c_M\}$.
Each feature $i\in\fml{F}$ takes values from a domain $\mbb{D}_i$.
Domains are ordinal and bounded, and each domain can be defined on
boolean, integer or real values. If $x_i\in\mbb{D}_i$, then
$\lambda(i)$ and $\mu(i)$ denote respectively the smallest and largest
values that $x_i$ can take, i.e.\ $\lambda(i)\le{x_i}\le\mu(i)$.
%
%Domains are bounded, and for each feature, $\lambda(i)$ denotes the
%smallest element of $\mbb{D}_i$ and $\mu(i)$ denotes the largest
%element of $\mbb{D}_i$.
%
Feature space is defined as
$\mbb{F}=\mbb{D}_1\times{\mbb{D}_2}\times\ldots\times{\mbb{D}_N}$.
The notation $\mbf{x}=(x_1,\ldots,x_N)$ denotes an arbitrary point in 
feature space, where each $x_i$ is a variable taking values from
$\mbb{D}_i$. Moreover, the notation $\mbf{v}=(v_1,\ldots,v_N)$
represents a specific point in feature space, where each $v_i$ is a
constant representing one concrete value from $\mbb{D}_i$.
An \emph{instance} (or example) denotes a pair $(\mbf{v}, c)$, where 
$\mbf{v}\in\mbb{F}$ and $c\in\fml{K}$. (We also use the term
\emph{instance} to refer to $\mbf{v}$, leaving $c$ implicit.)
An ML classifier $\mbb{C}$ is characterized by a \emph{classification
function} $\kappa$ that maps feature space $\mbb{F}$ into the set of 
classes $\fml{K}$, i.e.\ $\kappa:\mbb{F}\to\fml{K}$.

\paragraph{Monotonic classification.}
Given two points in feature space $\mbf{a}$ and $\mbf{b}$,
$\mbf{a}\le\mbf{b}$ if $a_i\le{b_i}$, for all $i\in\{1,\ldots,N\}$.
A set of classes $\fml{K}=\{c_1,\ldots,c_M\}$ is \emph{ordered} if it
respects a total order $\preccurlyeq$, with
$c_1\preccurlyeq{c_2}\preccurlyeq\ldots\preccurlyeq{c_M}$.
An ML classifier $\mbb{C}$ is fully monotonic if the associated
classification function is monotonic,
i.e.\ $\mbf{a}\le\mbf{b}\Rightarrow\kappa(\mbf{a})\preccurlyeq\kappa(\mbf{b})$\footnote{%
  The paper adopts the classification of monotonic classifiers
  proposed in earlier work~\cite{velikova-tnn10}.}.
Throughout the paper, when referring to a monotonic classifier, this
signifies a fully monotonic classifier.
In addition, the interaction with a classifier is restricted to
computing the value of $\kappa(\mbf{v})$, for some point
$\mbf{v}\in\mbb{F}$, i.e.\ the classifier will be viewed as a
black-box.

\begin{example}[Running example] \label{ex01:def}
   Let us consider a classifier for predicting student grades. We
   assume that the classifier has learned the following formula
   (after being trained with grades of students from different
   cohorts):
   \[%\small
   \begin{array}{rcl}
     S & = & \max\left[0.3 \times{Q}+0.6\times{X}+0.1\times{H},R\right]\\[2pt]
     M & = & \tn{ite}(S\ge9,A,\tn{ite}(S\ge7,B,\tn{ite}(S\ge5,C,\\[1.5pt]
     & & \tn{ite}(S\ge4,D,\tn{ite}(S\ge2,E,F)))))\\
   \end{array}
   \]
   $S$, $Q$, $X$, $H$ and $R$ denote, respectively, the final score,
   the marks on the quiz, the exam, the homework, and the 
   mark of an optional research project. Each mark ranges from 0 to
   10. (For the optional mark $R$, the final mark is 0 if the
   student opts out.)  The final score is the largest of the two
   marks, as shown above.
   Moreover, the final grade $M$ is defined using an $\tn{ite}$
   (if-then-else) operator, and ranges from $A$ to $F$.
   As a result, $Q$, $X$, $H$ and $R$ represent the features of the 
   classification problem, respectively numbered 1, 2, 3 and 4, and so
   $\fml{F}=\{1,2,3,4\}$.
   Each feature takes values from $[0,10]$, i.e. $\lambda(i)=0$ and
   $\mu(i)=10$. The set of classes is $\fml{K}=\{A,B,C,D,E,F\}$, with
   ${F}\preccurlyeq{E}\preccurlyeq{D}\preccurlyeq{C}\preccurlyeq{B}\preccurlyeq{A}$.
   Clearly, the complete classifier (that given the different marks
   computes a final grade) is monotonic.
   %
   %Throughout the paper, we will not have direct access to the
   %details of the classifier (i.e.\ we do not know how $M$ is
   %computed); we will only be able to consult the classifier by asking
   %for the value of $\kappa(q,e,h,r)$, where $q,e,h,r$ denote concrete
   %values from the features' domains.
   %
   Moreover, we will we consider a specific point of feature space 
   representing student $s_1$, $(Q,X,H,R)=(10,10,5,0)$, with a
   predicted grade of $A$, i.e.\ $\kappa(10,10,5,0)=A$.
\end{example}

\paragraph{Abductive and contrastive explanations.}
We now define formal explanations.
Prime implicant (PI) explanations~\cite{darwiche-ijcai18} denote a
minimal set of literals (relating a feature value $x_i$ and a constant
$v_i$ from its domain $\mbb{D}_i$) that are sufficient for the
prediction\footnote{%
PI-explanations are related with abduction, and so are also referred
to as abductive explanations (AXp)~\cite{inms-aaai19}. More recently,
PI-explanations have been studied from a knowledge compilation
perspective~\cite{marquis-kr20}.}.
Formally, given $\mbf{v}=(v_1,\ldots,v_N)\in\mbb{F}$ with $\kappa(\mbf{v})=c$, a PI-explanation
(AXp) is any minimal subset $\fml{X}\subseteq\fml{F}$ such that,
\begin{equation} \label{eq:axp}
  \forall(\mbf{x}\in\mbb{F}).
  \left[
    \bigwedge\nolimits_{i\in{\fml{X}}}(x_i=v_i)
    \right]
  \limply(\kappa(\mbf{x})=c)
\end{equation}
%%\label{page:eq:axp}
%
AXp's can be viewed as answering a 'Why?' question, i.e.\ why is some
prediction made given some point in feature space. A different view of
explanations is a contrastive explanation~\cite{miller-aij19}, which
answers a 'Why Not?' question, i.e.\ which features can be changed to
change the prediction. A formal definition of contrastive explanation
is proposed in recent work~\cite{inams-corr20}.
Given $\mbf{v}=(v_1,\ldots,v_N)\in\mbb{F}$ with $\kappa(\mbf{v})=c$, a CXp is any minimal
subset $\fml{Y}\subseteq\fml{F}$ such that,
\begin{equation} \label{eq:cxp}
  \exists(\mbf{x}\in\mbb{F}).\bigwedge\nolimits_{j\in\fml{F}\setminus\fml{Y}}(x_j=v_j)\land(\kappa(\mbf{x})\not=c) %\not\in\fml{Y}
\end{equation}
Building on the results of R.~Reiter in model-based
diagnosis~\cite{reiter-aij87},~\cite{inams-corr20} proves a minimal
hitting set (MHS) duality relation between AXp's and CXp's, i.e.\ AXp's are
MHSes of CXp's and vice-versa.

\begin{example}[AXp's \& CXp's] \label{ex01:xps}
   As can be readily observed (from the expression for $M$
   in~\autoref{ex01:def}), as long as $Q$ and $X$ take value 10,
   the prediction will be $A$, \emph{independently} of the values
   given to $H$ and $R$. Hence, given $(Q,X,H,R)=(10,10,5,0)$, one AXp
   is $\{1,2\}$.
   Moreover, to obtain a different prediction, it suffices to allow
   a suitable change of value in $Q$ (or alternatively in $X$). Hence,
   given $(Q,X,H,R)=(10,10,5,0)$, one CXp is $\{1\}$ (and another is
   $\{2\}$).
   As can be observed, $\{1,2\}$ is the only MHS of $\{\{1\},\{2\}\}$
   and vice-versa. These are the only AXp's and CXp's for the example
   instance. %\\
   %Finally, observe that the details of how the classifier is
   %implemented are not important for the results in this paper. We
   %just need to require that it is monotonic (and we will assume that
   %the classifier correctly predicts the final marks given the
   %student's grades.)
\end{example}

\begin{comment}
%
As described in~\autoref{sec:xps}, our analysis does not need to know
how the actual prediction is made. It just needs to test how the ML
model answers to changes in the feature values. Hence, our analysis is
independent of the model's details; we say that the approach is
model-agnostic (similarly to earlier
work~\cite{guestrin-kdd16,lundberg-nips17,guestrin-aaai18}) but it
nevertheless enables us to compute AXp's and/or
CXp's~\cite{darwiche-ijcai18,miller-aij19,inams-corr20}.
%
\end{comment}

\paragraph{Boolean satisfiability (SAT).}
SAT is the decision problem for propositional logic. The paper uses
standard notation and definitions e.g.~\cite{handbook09}.
A propositional formula is defined on a set $U$ of boolean variables,
where the domain of each variable $u_i\in{U}$ is $\{0,1\}$. We
consider conjunctive normal form (CNF) formulas, where a formula is a
conjunction of clauses, each clause is a disjunction of literals, and
a literal is a variable $u_i$ or its negation $\neg{u_i}$.
CNF formulas and SAT reasoners are used in~\autoref{ssec:expm}.

\section{Explanations for Monotonic Classifiers} \label{sec:xps}

This section describes three algorithms.
The first algorithm serves to compute one AXp (and is referred to as
$\axp$). Its complexity is polynomial in the run time complexity of
the classifier.
The second algorithm serves to compute one CXp (and is referred to as
$\cxp$). It has the same polynomial complexity as $\axp$.
The third algorithm shows how to use SAT reasoners for iteratively
enumerating AXp's or CXp's. This algorithm is inspired by earlier
work~\cite{lpmms-cj16}, but with key observations that minimize the
number of times a SAT reasoner is called. This algorithm is based on
the other two algorithms, and is described in~\autoref{ssec:expm}.

One key property of the three algorithms is that, besides knowing that
the classifier is monotonic, \emph{no} additional information about
the classifier is required. Indeed, the algorithms described in this
section only require running the classifier for specific points in
feature space.
Thus, and similarly to LIME~\cite{guestrin-kdd16},
SHAP~\cite{lundberg-nips17} or Anchor~\cite{guestrin-aaai18}, the
algorithms proposed in this section are model-agnostic. However, and
in contrast also with LIME, SHAP or Anchor, the proposed algorithms
compute rigorously defined AXp's, CXp's, and also
serve for the enumeration of explanations.

\subsection{Finding One AXp and One CXp}

The two algorithms $\axp$ and $\cxp$ (shown as~\autoref{alg:axp} and
\autoref{alg:cxp}) share a number of common concepts, while solving
different problems. These concepts are summarized next.
%
%
%%The algorithms $\axp$ and $\cxp$ (shown as \autoref{alg:acp
%%\autoref{alg:axp} shows the proposed approach for finding one AXp.
%
\begin{algorithm}[t]
  \caption{Finding one AXp -- $\axp(\fml{F},\fml{S},\mbf{v})$} \label{alg:axp}
  \begin{algorithmic}[1]
  \STATE{%
    \label{alg:axp:ln01} $\mbf{v}_{\tn{L}}\gets(v_1,\ldots,v_N)$}
  \STATE{%
    \label{alg:axp:ln02} $\mbf{v}_{\tn{U}}\gets(v_1,\ldots,v_N)$} \COMMENT{Ensures: $\kappa(\mbf{v}_{\tn{L}})=\kappa(\mbf{v}_{\tn{U}})$}
  \STATE{%
    \label{alg:axp:ln03} $(\fml{C},\fml{D},\fml{P})\gets(\fml{F},\emptyset,\emptyset)$}
  \FORALL{\label{alg:axp:ln04} $i\in\fml{S}$}
  \STATE{\label{alg:axp:ln05} $(\mbf{v}_{\tn{L}},\mbf{v}_{\tn{U}},\fml{C},\fml{D})\gets\mkuniv(i,\mbf{v},\mbf{v}_{\tn{L}},\mbf{v}_{\tn{U}},\fml{C},\fml{D})$}
  %\STATE{$\mbf{v}_{\tn{L}}\gets(v_{\tn{L}_1},\ldots,\lambda(i),\ldots,v_{\tn{L}_N})$}
  %\STATE{$\mbf{v}_{\tn{U}}\gets(v_{\tn{U}_1},\ldots,\mu(i),\ldots,v_{\tn{U}_N})$}
  %\STATE{$(\fml{C},\fml{D})\gets(\fml{C}\setminus\{i\},\fml{D}\cup\{i\})$}
  \label{alg:axp:ln06}\ENDFOR\COMMENT{Require: $\kappa(\mbf{v}_{\tn{L}})=\kappa(\mbf{v}_{\tn{U}})$, given $\fml{S}$}
  \FORALL[Loop~inv.:~$\kappa(\mbf{v}_{\tn{L}})=\kappa(\mbf{v}_{\tn{U}})$]{%
    \label{alg:axp:ln07} $i\in\fml{F}\setminus\fml{S}$}
  \STATE{\label{alg:axp:ln08} $(\mbf{v}_{\tn{L}},\mbf{v}_{\tn{U}},\fml{C},\fml{D})\gets\mkuniv(i,\mbf{v},\mbf{v}_{\tn{L}},\mbf{v}_{\tn{U}},\fml{C},\fml{D})$}
  %\STATE{$\mbf{v}_{\tn{L}}\gets(v_{\tn{L}_1},\ldots,\lambda(i),\ldots,v_{\tn{L}_N})$}
  %\STATE{$\mbf{v}_{\tn{U}}\gets(v_{\tn{U}_1},\ldots,\mu(i),\ldots,v_{\tn{U}_N})$}
  %\STATE{$(\fml{C},\fml{D})\gets(\fml{C}\setminus\{i\},\fml{D}\cup\{i\})$}
  \IF[If invariant broken, fix it]{\label{alg:axp:ln09} $\kappa(\mbf{v}_{\tn{L}})\not=\kappa(\mbf{v}_{\tn{U}})$}
  \STATE{\label{alg:axp:ln10} $(\mbf{v}_{\tn{L}},\mbf{v}_{\tn{U}},\fml{D},\fml{P})\gets\mkfix(i,\mbf{v},\mbf{v}_{\tn{L}},\mbf{v}_{\tn{U}},\fml{D},\fml{P})$}
  %\STATE{$(\fml{D},\fml{P})\gets(\fml{D}\setminus\{i\},\fml{P}\cup\{i\})$}
  %\STATE{$\mbf{v}_{\tn{L}}\gets(v_{\tn{L}_1},\ldots,v_i,\ldots,v_{\tn{L}_N})$}
  %\STATE{$\mbf{v}_{\tn{U}}\gets(v_{\tn{U}_1},\ldots,v_i,\ldots,v_{\tn{U}_N})$}
  \ENDIF
  \ENDFOR
  \RETURN{%
    \label{alg:axp:ln13} $\fml{P}$}
\end{algorithmic}

\end{algorithm}
%
%
%%\autoref{alg:axp} shows the proposed approach for finding one AXp.
%
\begin{algorithm}[t]
  \caption{Finding one CXp -- $\cxp(\fml{F},\fml{S},\mbf{v})$} \label{alg:cxp}
  \begin{algorithmic}[1]
  \STATE{%
    \label{alg:cxp:ln01} $\mbf{v}_{\tn{L}}\gets(\lambda(1),\ldots,\lambda(N))$}
  \STATE{%
    \label{alg:cxp:ln02} $\mbf{v}_{\tn{U}}\gets(\mu(1),\ldots,\mu(N))$}\COMMENT{Ensures: $\kappa(\mbf{v}_{\tn{L}})\not=\kappa(\mbf{v}_{\tn{U}})$}
  \STATE{%
    \label{alg:cxp:ln03} $(\fml{C},\fml{D},\fml{P})\gets(\fml{F},\emptyset,\emptyset)$}
  \FORALL{\label{alg:cxp:ln04} $i\in\fml{S}$}
  \STATE{%
    \label{alg:cxp:ln05} $(\mbf{v}_{\tn{L}},\mbf{v}_{\tn{U}},\fml{C},\fml{D})\gets\mkfix(i,\mbf{v},\mbf{v}_{\tn{L}},\mbf{v}_{\tn{U}},\fml{C},\fml{D})$}
  %\STATE{$\mbf{v}_{\tn{L}}\gets(v_{\tn{L}_1},\ldots,v_i,\ldots,v_{\tn{L}_N})$}
  %\STATE{$\mbf{v}_{\tn{U}}\gets(v_{\tn{U}_1},\ldots,v_i,\ldots,v_{\tn{U}_N})$}
  %\STATE{$(\fml{C},\fml{D})\gets(\fml{C}\setminus\{i\},\fml{D}\cup\{i\})$}
  \label{alg:cxp:ln06} \ENDFOR\COMMENT{Require: $\kappa(\mbf{v}_{\tn{L}})\not=\kappa(\mbf{v}_{\tn{U}})$, given $\fml{S}$}
  \FORALL[Loop~inv.:~$\kappa(\mbf{v}_{\tn{L}})\not=\kappa(\mbf{v}_{\tn{U}})$]{%
    \label{alg:cxp:ln07} $i\in\fml{F}\setminus\fml{S}$}
  \STATE{\label{alg:cxp:ln08} $(\mbf{v}_{\tn{L}},\mbf{v}_{\tn{U}},\fml{C},\fml{D})\gets\mkfix(i,\mbf{v},\mbf{v}_{\tn{L}},\mbf{v}_{\tn{U}},\fml{C},\fml{D})$}
  %\STATE{$\mbf{v}_{\tn{L}}\gets(v_{\tn{L}_1},\ldots,v_i,\ldots,v_{\tn{L}_N})$}
  %\STATE{$\mbf{v}_{\tn{U}}\gets(v_{\tn{U}_1},\ldots,v_i,\ldots,v_{\tn{U}_N})$}
  %\STATE{$(\fml{C},\fml{D})\gets(\fml{C}\setminus\{i\},\fml{D}\cup\{i\})$}
  \IF[If invariant broken, fix it]{\label{alg:cxp:ln09} $\kappa(\mbf{v}_{\tn{L}})=\kappa(\mbf{v}_{\tn{U}})$}
  \STATE{\label{alg:cxp:ln10} $(\mbf{v}_{\tn{L}},\mbf{v}_{\tn{U}},\fml{D},\fml{P})\gets\mkuniv(i,\mbf{v},\mbf{v}_{\tn{L}},\mbf{v}_{\tn{U}},\fml{D},\fml{P})$}
  %\STATE{$(\fml{D},\fml{P})\gets(\fml{D}\setminus\{i\},\fml{P}\cup\{i\})$}
  %\STATE{$\mbf{v}_{\tn{L}}\gets(v_{\tn{L}_1},\ldots,\lambda(i),\ldots,v_{\tn{L}_N})$}
  %\STATE{$\mbf{v}_{\tn{U}}\gets(v_{\tn{U}_1},\ldots,\mu(i),\ldots,v_{\tn{U}_N})$}
  \ENDIF
  \ENDFOR
  \RETURN{%
    \label{alg:cxp:ln13} $\fml{P}$}
\end{algorithmic}

\end{algorithm}
The two algorithms iteratively update three sets of features
($\fml{C}$, $\fml{D}$ and $\fml{P}$) and two points in feature
space ($\mbf{v}_{\tn{L}}$ and $\mbf{v}_{\tn{U}}$).
Using these variables, the two algorithms maintain two invariants.
The first invariant is that $\fml{C}$, $\fml{D}$ and $\fml{P}$ form a
partition of $\fml{F}$, and represent respectively the candidate,
dropped and picked sets of features (with the picked features denoting
those that are included either in an AXp or an CXp).
The second invariant serves to ensure that the selected set of
features satisfies \eqref{eq:axp} (for $\axp$) or \eqref{eq:cxp} (for
$\cxp$).
Maintaining this invariant, requires iteratively updating two points
$\mbf{v}_{\tn{L}}=(v_{\tn{L}_1},\ldots,v_{\tn{L}_N})$ and
$\mbf{v}_{\tn{U}}=(v_{\tn{U}_1},\ldots,v_{\tn{U}_N})$, denoting
respectively lower and upper bounds on the class values that can be
obtained given the features that are allowed to take any value in
their domain.
%As shown later, for $\axp$ the invariant is
%$\kappa(\mbf{v}_{\tn{L}}=\mbf{v}_{\tn{U}})$, and for $\cxp$ the
%invariant is
%$\kappa(\mbf{v}_{\tn{L}})\not=\kappa(\mbf{v}_{\tn{U}})$. In the first
%case, the set of literals that ...

%%%%%%%%%%%%%%%%%%%%%%%%%%%%%%%%%%%%%%%%%%%%%%%%%%%%%%%%%%%%%%%%%%%%%%%%%%%%%%%%
%% Moving table around... Used in  \autoref{ex01:axpm}
\begin{table*}[t]
  \begin{center}
    \scalebox{1.0}{ %0.9125
      \renewcommand{\tabcolsep}{0.35em}
      \begin{tabular}{|c|cc|cc|cc|c|cc|} \toprule
        \multirow{2}{*}{Feat.} &
        \multicolumn{2}{c|}{Initial values} &
        \multicolumn{2}{c|}{Changed values} &
        \multicolumn{2}{c|}{Predictions} &
        \multirow{2}{*}{Dec.} &
        \multicolumn{2}{c|}{Resulting values} \\
        \cline{2-7}\cline{9-10}%\midrule
        &
        $\mbf{v}_{\tn{L}}$ & $\mbf{v}_{\tn{U}}$ & 
        $\mbf{v}_{\tn{L}}$ & $\mbf{v}_{\tn{U}}$ & 
        $\kappa(\mbf{v}_{\tn{L}})$ & $\kappa(\mbf{v}_{\tn{U}})$ & 
        &
        $\mbf{v}_{\tn{L}}$ & $\mbf{v}_{\tn{U}}$
        \\ \toprule
        1 &
        $(10{,}10{,}5{,}0)$ & $(10{,}10{,}5{,}0)$ &
        $(0{,}10{,}5{,}0)$ & $(10{,}10{,}5{,}0)$ &
        $C$ & $A$ &
        \keepmark &  %keep 1
        $(10{,}10{,}5{,}0)$ & $(10{,}10{,}5{,}0)$
        \\ \midrule
        2 &
        $(10{,}10{,}5{,}0)$ & $(10{,}10{,}5{,}0)$ &
        $(10{,}0{,}5{,}0)$ & $(10{,}10{,}5{,}0)$ &
        $E$ & $A$ &
        \keepmark &  %keep 2
        $(10{,}10{,}5{,}0)$ & $(10{,}10{,}5{,}0)$
        \\ \midrule
        3 &
        $(10{,}10{,}5{,}0)$ & $(10{,}10{,}5{,}0)$ &
        $(10{,}10{,}0{,}0)$ & $(10{,}10{,}5{,}0)$ &
        $A$ & $A$ &
        \dropmark &  %drop 3 
        $(10{,}10{,}0{,}0)$ & $(10{,}10{,}10{,}0)$
        \\ \midrule
        4 &
        $(10{,}10{,}0{,}0)$ & $(10{,}10{,}10{,}0)$ &
        $(10{,}10{,}0{,}0)$ & $(10{,}10{,}10{,}10)$ &
        $A$ & $A$ &
        \dropmark &  %drop 4
        $(10{,}10{,}0{,}0)$ & $(10{,}10{,}10{,}10)$
        \\
        \bottomrule
      \end{tabular}
    }
  \end{center}
  \caption{Execution of algorithm for finding one AXp} \label{tab:axp-ex}
\end{table*}

%%%%%%%%%%%%%%%%%%%%%%%%%%%%%%%%%%%%%%%%%%%%%%%%%%%%%%%%%%%%%%%%%%%%%%%%%%%%%%%%

%\subsection{Finding One Abductive Explanation}
%\subsection{Finding One Abductive Explanation}
\paragraph{Finding one AXp.}

We detail below the main steps of algorithm
$\axp$ (see  \autoref{alg:axp}). (Lines~\ref{alg:axp:ln04}
to~\ref{alg:axp:ln06} are used for enumerating explanations, and so we
assume $\fml{S}=\emptyset$ for now.)
The main goal of $\axp$ is to find a \emph{maximal} set of features
$\fml{D}$ which are allowed to take \emph{any} value, i.e.\ that are
\emph{free}. For such a set $\fml{D}$, the set of features that remain
fixed to the value specified in $\mbf{v}$,
i.e.\ $\fml{P}=\fml{F}\setminus\fml{D}$, is a minimal set of (picked)
features that is sufficient for the prediction, as intended.
The different sets used by the algorithm are initialized
in~\autoref{alg:axp:ln03}. (As noted earlier, the sets
$\fml{C}$, $\fml{D}$ and $\fml{P}$ form a partition of $\fml{F}$, and
$\fml{C}=\emptyset$ upon termination.)

For $\axp$, the second invariant of the algorithm is that
$\kappa(\mbf{v}_{\tn{L}})=\kappa(\mbf{v}_{\tn{U}})$, i.e.\ by allowing
the features in $\fml{P}\cup\fml{C}$ to take the corresponding value
in $\mbf{v}$, the value of the prediction is guaranteed not to
change.

The use of the second invariant
$\kappa(\mbf{v}_{\tn{L}})=\kappa(\mbf{v}_{\tn{U}})$ is justified by
the following result.
\begin{proposition} \label{prop:axp01}
  If $\kappa(\mbf{v}_{\tn{L}})=\kappa(\mbf{v}_{\tn{U}})$, then it
  holds that,
  $\forall(\mbf{x}\in\mbb{F}).
           [\mbf{v}_{\tn{L}}\le\mbf{x}\le\mbf{v}_{\tn{U}}]\rightarrow
           [\kappa(\mbf{x})=\kappa(\mbf{v})]$.
\end{proposition}

The algorithm starts by enforcing the second invariant as the result
of executing lines \ref{alg:axp:ln01} and \ref{alg:axp:ln02}.

Moreover, $\axp$ analyzes one feature at a time. Starting from the set
$\fml{C}$ of candidate features (in~\autoref{alg:axp:ln07}), the
algorithm iteratively picks a feature $i$ from $\fml{C}$ and makes a
decision about whether to drop the feature from the explanation.
The first step is to assume that the feature $i$ can indeed be allowed
to take any value. This is done in~\autoref{alg:axp:ln08}, by calling
the following function $\mkuniv$:

\begin{algorithmic}
  \STATE{$\mbf{v}_{\tn{L}}\gets(v_{\tn{L}_1},\ldots,\lambda(i),\ldots,v_{\tn{L}_N})$}
  \STATE{$\mbf{v}_{\tn{U}}\gets(v_{\tn{U}_1},\ldots,\mu(i),\ldots,v_{\tn{U}_N})$}
  \STATE{$(\fml{A},\fml{B})\gets(\fml{A}\setminus\{i\},\fml{B}\cup\{i\})$}
  \RETURN{$(\mbf{v}_{\tn{L}},\mbf{v}_{\tn{U}},\fml{A},\fml{B})$}
\end{algorithmic}

where $\fml{A}$ is replaced by $\fml{C}$ and $B$ is replaced by
$\fml{D}$, and so feature $i$ is moved from $\fml{C}$ to $\fml{D}$.
In addition, the value of $i$ is now allowed to range from
$\lambda(i)$ (in $\mbf{v}_{\tn{L}}$) to $\mu(i)$ (in
$\mbf{v}_{\tn{U}}$),
%
%and so the values of $\mbf{v}_{\tn{L}}$ and $\mbf{v}_{\tn{U}}$ are
%updated accordingly.

The next step of the algorithm (in~\autoref{alg:axp:ln09}) is to
decide whether allowing $i$ to take any value breaks the invariant
$\kappa(\mbf{v}_{\tn{L}})=\kappa(\mbf{v}_{\tn{U}})$.
If the invariant is not broken, then the algorithm moves to
analyze the next feature (in~\autoref{alg:axp:ln07}). However, if the
invariant is broken, then the the feature cannot take any value, and
so it must be fixed to the corresponding value in $\mbf{v}$.
This is done by calling (in~\autoref{alg:axp:ln10}) the following
function $\mkfix$:

\begin{algorithmic}
  \STATE{$\mbf{v}_{\tn{L}}\gets(v_{\tn{L}_1},\ldots,v_i,\ldots,v_{\tn{L}_N})$}
  \STATE{$\mbf{v}_{\tn{U}}\gets(v_{\tn{U}_1},\ldots,v_i,\ldots,v_{\tn{U}_N})$}
  \STATE{$(\fml{A},\fml{B})\gets(\fml{A}\setminus\{i\},\fml{B}\cup\{i\})$}
  \RETURN{$(\mbf{v}_{\tn{L}},\mbf{v}_{\tn{U}},\fml{A},\fml{B})$}
\end{algorithmic}

where $\fml{A}$ is replaced by $\fml{D}$ and $\fml{B}$ is replaced by 
$\fml{P}$, and so feature $i$ is moved from $\fml{D}$ to $\fml{P}$.
In addition, the value of $i$ is once again fixed to the corresponding
value in $\mbf{v}$.
After analyzing all features, the algorithm $\axp$ terminates
(in~\autoref{alg:axp:ln13}) by return the (minimal) set of features 
$\fml{P}$ that are fixed to their value in $\mbf{v}$.
It is immediate to conclude that each feature is analyzed once, and
that for each feature, the classifier is invoked twice.
Given the discussion above, we conclude that,
\begin{theorem}
  Given a monotonic classifier, an instance $\mbf{v}$ with prediction
  $c=\kappa(\mbf{v})$, \autoref{alg:axp} computes one AXp in linear
  time in the running time complexity of the classifier.
\end{theorem}

We illustrate the operation of $\axp$, with an example.

\begin{example} \label{ex01:axpm}
  Given the monotonic classifier from~\autoref{ex01:def}, and the
  concrete case of student $s_1$, with $(Q,X,H,R)=(10,10,5,0)$ and
  predicted mark $A$, we show how one PI-explanation can computed.
  (In settings with more than one AXp, changing the order of how
  features are analyzed, may results in a different explanation being
  obtained.)
  For each feature $i$, $1\le{i}\le4$, $\lambda(i)=0$ and $\mu(i)=10$.
  Moreover, features are analyzed in order:
  $\langle1,2,3,4\rangle$; the order is arbitrary.
  The algorithm's execution is summarized in~\autoref{tab:axp-ex}.
  As can be observed, features 1 and 2 are kept as part of the
  PI-explanation (decision is \keepmark in \autoref{alg:axp:ln09},
  i.e.\ invariant is broken and features are kept),
  whereas features 3 and 4 are dropped from the PI-explanation
  (decision is \dropmark, i.e.\ invariant holds). 
  As a result, the PI-explanation for the grade of student $s_1$ is
  $\{1,2\}$, which denotes that as long as $(Q=10)\land(X=10)$, the
  prediction will be $A$.
\end{example}

\begin{table*}
  \begin{center}
    \scalebox{1.0}{ %0.9125
      \renewcommand{\tabcolsep}{0.35em}
      \begin{tabular}{|c|cc|cc|cc|c|cc|} \toprule
        \multirow{2}{*}{Feat.} &
        \multicolumn{2}{c|}{Initial values} &
        \multicolumn{2}{c|}{Changed values} &
        \multicolumn{2}{c|}{Predictions} &
        \multirow{2}{*}{Dec.} &
        \multicolumn{2}{c|}{Resulting values} \\
        \cline{2-7}\cline{9-10}%\midrule
        &
        $\mbf{v}_{\tn{L}}$ & $\mbf{v}_{\tn{U}}$ & 
        $\mbf{v}_{\tn{L}}$ & $\mbf{v}_{\tn{U}}$ & 
        $\kappa(\mbf{v}_{\tn{L}})$ & $\kappa(\mbf{v}_{\tn{U}})$ & 
        &
        $\mbf{v}_{\tn{L}}$ & $\mbf{v}_{\tn{U}}$
        \\ \toprule
        1 &
        $(0{,}0{,}0{,}0)$ & $(10{,}10{,}10{,}10)$ &
        $(10{,}0{,}0{,}0)$ & $(10{,}10{,}10{,}10)$ &
        $E$ & $A$ &
        \dropmark & % drop 1 ; fixed
        $(10{,}0{,}0{,}0)$ & $(10{,}10{,}10{,}10)$
        \\ \midrule
        2 &
        $(10{,}0{,}0{,}0)$ & $(10{,}10{,}10{,}10)$ &
        $(10{,}10{,}0{,}0)$ & $(10{,}10{,}10{,}10)$ &
        $A$ & $A$ &
        \keepmark & % keep 2 ; universal
        $(10{,}0{,}10{,}0)$ & $(10{,}10{,}10{,}10)$
        \\ \midrule
        3 &
        $(10{,}0{,}0{,}0)$ & $(10{,}10{,}10{,}10)$ &
        $(10{,}0{,}5{,}0)$ & $(10{,}10{,}5{,}10)$ &
        $E$ & $A$ &
        \dropmark & % drop 3 ; fixed
        $(10{,}0{,}5{,}0)$ & $(10{,}0{,}5{,}10)$
        \\ \midrule
        4 &
        $(10{,}0{,}5{,}0)$ & $(10{,}10{,}5{,}10)$ &
        $(10{,}0{,}5{,}0)$ & $(10{,}10{,}5{,}0)$ &
        $E$ & $A$ &
        \dropmark & % drop 4 ; fixed
        $(10{,}0{,}5{,}0)$ & $(10{,}10{,}5{,}0)$
        \\
        \bottomrule
      \end{tabular}
    }
  \end{center}
  \caption{Execution of algorithm for finding one CXp} \label{tab:cxp-ex}
\end{table*}

%%%%%%%%%%%%%%%%%%%%%%%%%%%%%%%%%%%%%%%%%%%%%%%%%%%%%%%%%%%%%%%%%%%%%%%%%%%%%%%%

\paragraph{Finding one CXp.}

The two algorithms $\axp$ and $\cxp$ are organized in a similar way.
(This in part results from the fact that AXps are minimal hitting sets
of CXps and vice-versa~\cite{inams-corr20}.)
We briefly explain the differences when computing a CXp (see
\autoref{alg:cxp}). (Lines~\ref{alg:cxp:ln04} to~\ref{alg:cxp:ln06}
are used for enumerating explanations, and so we assume
$\fml{S}=\emptyset$ for now.)

The main goal of $\cxp$ is to find a \emph{maximal} set of features
$\fml{D}$ that are only allowed to take the value specified in
$\mbf{v}$, i.e.\ that are \emph{fixed}. For such a set $\fml{D}$, the
set of features that are allowed to take any value,
i.e.\ $\fml{P}=\fml{F}\setminus\fml{D}$, is a minimal set that, by
being allowed to take any value in their domain, suffices for allowing
the prediction to change, as intended.
The different sets used by the algorithm are initialized
in~\autoref{alg:axp:ln03}. 

For $\cxp$, the second invariant of the algorithm is that
$\kappa(\mbf{v}_{\tn{L}})\not=\kappa(\mbf{v}_{\tn{U}})$, i.e.\ by allowing
the features in $\fml{P}\cup\fml{C}$ to take any value, the value of
the prediction does not change.
The algorithm starts by enforcing the second invariant as the result
of executing lines \ref{alg:axp:ln01} and \ref{alg:axp:ln02}.

The use of the second invariant
$\kappa(\mbf{v}_{\tn{L}})\not=\kappa(\mbf{v}_{\tn{U}})$ is justified
by the following result.

\begin{proposition} \label{prop:cxp01}
  If $\kappa(\mbf{v}_{\tn{L}})\not=\kappa(\mbf{v}_{\tn{U}})$, then it
  holds that,
  $\exists(\mbf{x}\in\mbb{F}).
          [\mbf{v}_{\tn{L}}\le\mbf{x}\le\mbf{v}_{\tn{U}}]\land
          [\kappa(\mbf{x})\not=\kappa(\mbf{v})]$.
\end{proposition}

Similarly to $\axp$, $\cxp$ analyzes one feature at a time. Starting
from the set $\fml{C}$ of candidate features
(in~\autoref{alg:cxp:ln07}), the algorithm iteratively picks a feature
$i$ from $\fml{C}$ and makes a 
decision about whether to drop the feature from the explanation.
The first step is to assume that the feature $i$ can indeed be fixed
to the corresponding value in $\mbf{v}$. This is done
in~\autoref{alg:cxp:ln08}, by calling the following function
$\mkfix$,
where $\fml{A}$ is replaced by $\fml{C}$, and $B$ is replaced by
$\fml{D}$, and so feature $i$ is moved from $\fml{C}$ to $\fml{D}$.
In addition, the value of $i$ is now fixed to its value in $\mbf{v}$.

The next step of the algorithm (in~\autoref{alg:cxp:ln09}) is to
decide whether fixing the value of  $i$ breaks the invariant
$\kappa(\mbf{v}_{\tn{L}})\not=\kappa(\mbf{v}_{\tn{U}})$.
If the invariant is not broken, then the algorithm moves to
analyze the next feature (in~\autoref{alg:cxp:ln07}). However, if the
invariant is broken, then the feature cannot be fixed, and so it must
be allowed to take any value from its domain.
This is done by calling (in~\autoref{alg:cxp:ln10}) the following
function $\mkuniv$, with
$\fml{A}$ replaced by $\fml{D}$ and $\fml{B}$ replaced by 
$\fml{P}$, and so feature $i$ is moved from $\fml{D}$ to $\fml{P}$.
In addition, the value of $i$ is once again allowed to take any value
from its domain.
After analyzing all features, the algorithm $\cxp$ terminates
(in~\autoref{alg:cxp:ln13}) by returning the (minimal) set of features 
$\fml{P}$ that are allowed to take any value from their domain.
It is immediate to conclude that each feature is analyzed once, and
that for each feature, the classifier is invoked twice.
Given the discussion above, we conclude that,
\begin{theorem}
  Given a monotonic classifier, an instance $\mbf{v}$ with prediction
  $c=\kappa(\mbf{v})$, \autoref{alg:cxp} computes one CXp in linear
  time in the running time complexity of the classifier.
\end{theorem}

We illustrate the operation of $\cxp$, with an example.

\begin{example} \label{ex01:cxpm}
  For the running example (see Examples \ref{ex01:def}, \ref{ex01:xps}
  and \ref{ex01:axpm}), for instance $\mbf{v}_0=(10,10,5,0)$ with
  prediction $A$, we illustrate the computation of one CXp. The
  algorithm's execution is summarized in~\autoref{tab:cxp-ex}.
  (When computing one CXp, a feature is \emph{kept} (decision is
  \keepmark) if it is declared free, and it is \emph{dropped}
  (decision is \dropmark) if it must be fixed.) 
  As can be observed, a contrastive explanation is: $\{2\}$,
  i.e.\ there is an assignment to feature 2 (i.e.\ to $X$), which
  guarantees a change of prediction when the other features are kept
  to their values. For example, by setting $X=0$ (and keeping the
  remaining values fixed), the value of the prediction changes.
  %One can envision algorithms that search for assignments to $E$,
  %which change the  prediction, but which exhibit other properties.
  %
\end{example}

\paragraph{Complexity.}
As can be readily concluded from \autoref{alg:axp} and
\autoref{alg:cxp}, the algorithms execute in linear time in the number
of features. However, in each iteration of the algorithm, the
classifier is invoked twice, for finding the predicted classes for
$\mbf{v}_{\tn{L}}$ and for $\mbf{v}_{\tn{U}}$.
We will represent the time required by the classifier as
$\fml{T}_{\mbb{C}}$, and so the overall run time of each algorithm is
$\fml{O}(|\fml{F}|\times\fml{T}_{\mbb{C}})$.

\subsection{Enumerating Explanations} \label{ssec:expm}

We first show that for monotonic classifiers,
the enumeration of explanations with polynomial-time
delay is computationally hard. %an unrealistic aim.

\begin{theorem} \label{thm:exp01}
  Determining the existence of $\lfloor N/2 \rfloor{+}1$ AXp's
  (or CXp's) of a monotonic $N$-feature classifier is NP-complete.
\end{theorem}
(The proof is included in the appendix.)
Since the enumeration of AXp's and CXp's with polynomial delay is
unlikely, we describe in this section how to use SAT reasoners for 
the enumeration of AXp's and CXp's of a monotonic classifier.
%
%We now show how we can use SAT reasoners
%%oracles (i.e.\ witness-producing NP oracles)
%for iteratively enumerating AXp's and CXp's of a monotonic
%classifier.
%
(Although we prove the algorithm to be sound and complete, the
algorithm necessarily has leeway in selecting the order in which AXp's
and CXp's are listed.)

The algorithm uses the following propositional representation: %logic 
\begin{enumerate}[nosep,topsep=0pt,partopsep=0pt]
\item The algorithm will iteratively add clauses to a CNF formula
  $\fml{H}$. The clauses in $\fml{H}$ account for the AXp's and CXp's
  already computed, and serve to prevent their repetition.
\item Formula $\fml{H}$ is defined on a set of variables $u_i$,
  $1\le{i}\le{n}$, where each $u_i$ denotes whether feature $i$ is
  declared free ($u_i=1$) or is alternatively declared fixed
  ($u_i=0$).
\end{enumerate}
The algorithm proposed in this section requires exactly one call to a
SAT reasoner before computing one explanation (either AXp/CXp),
and one additional call to decide that all explanations have been
computed. As a result, the number of calls to a SAT reasoner is
$|\tn{AXp}|+|\tn{CXp}|+1$. Furthermore, the size of the formula grows
by one clause after each AXp or CXp is computed. In practice, for a
wide range of ML settings, both the number of variables and the number
of clauses are well within the reach of modern SAT reasoners.
\begin{proposition} \label{prop:exp01}
  Let $\mbf{v}$ be a point in feature space,
  let $\kappa(\mbf{v})=c\in\fml{K}$, and let
  $\fml{Z}\subseteq\fml{F}$. Then, either \eqref{eq:axp}
  (on page~\pageref{eq:axp}) holds, with
  $\fml{X}=\fml{Z}$, or \eqref{eq:cxp} (also on page~\pageref{eq:axp})
  holds, with   $\fml{Y}=\fml{F}\setminus\fml{Z}$, but not both.
\end{proposition}
\autoref{prop:exp01} essentially states that, given a set $\fml{Z}$ of
features, if these are fixed, and the others are allowed to take any
value from their domains, then either the prediction never changes, or
there exists an assignment to the non-fixed features, which causes the
prediction to change.
The approach for enumerating AXp's and CXp's is shown
in~\autoref{alg:exp}.
\begin{algorithm}[t]
  \caption{Enumeration of AXp's and CXp's} \label{alg:exp}
  \begin{algorithmic}[1]
  \STATE{%
    \label{alg:exp:ln01} $\fml{H}\gets\emptyset$}%
  \COMMENT{$\fml{H}$ defined on set $U$}
  \REPEAT\label{alg:exp:ln02} 
  \STATE{%
    \label{alg:exp:ln03} $(\outc,\mbf{u})\gets\SAT(\fml{H})$}
  \IF{\label{alg:exp:ln04} $\outc=\TRUE$}
  \STATE{%
    \label{alg:exp:ln05}%
    $\mbf{v}_{\tn{L}}\gets(v_{\tn{L}_1},\ldots,v_{\tn{L}_N}),\:\tn{s.t.}\:%|\,
    v_{\tn{L}_i}\gets\tn{ite}(u_i,\lambda(i),v_i)$}%,i=1,\ldots,N
  \STATE{%
    \label{alg:exp:ln06}%
    $\mbf{v}_{\tn{U}}\gets(v_{\tn{U}_1},\ldots,v_{\tn{U}_N}),\:\!\tn{s.t.}\;\!%|\,
    v_{\tn{U}_i}\gets\tn{ite}(u_i,\mu(i),v_i)$} %,i=1,\ldots,N
  \IF{\label{alg:exp:ln07}$\kappa(\mbf{v}_{\tn{L}})=\kappa(\mbf{v}_{\tn{U}})$}
  \STATE{\label{alg:exp:ln08}$\fml{S}\gets\{i\in\fml{F}\,|\,u_i=1\}$}%
  \COMMENT{$\fml{F}\setminus\fml{S}\supseteq$ some AXp}
  \STATE{\label{alg:exp:ln09}$\fml{P}\gets\axp(\fml{F},\fml{S},\mbf{u})$}
  \STATE{\label{alg:exp:ln10}$\prtaxp(\fml{P})$}
  \STATE{\label{alg:exp:ln11}$\fml{H}\gets\fml{H}\cup\{(\lor_{i\in\fml{P}}u_i)\}$}
  \ELSE\label{alg:exp:ln12}
  \STATE{\label{alg:exp:ln13}$\fml{S}\gets\{i\in\fml{F}\,|\,u_i=0\}$}%
  \COMMENT{$\fml{F}\setminus\fml{S}\supseteq$ some CXp}
  \STATE{\label{alg:exp:ln14}$\fml{P}\gets\cxp(\fml{F},\fml{S},\mbf{u})$}
  \STATE{\label{alg:exp:ln15}$\prtcxp(\fml{P})$}
  \STATE{\label{alg:exp:ln16}$\fml{H}\gets\fml{H}\cup\{(\lor_{i\in\fml{P}}\neg{u_i})\}$}
  \ENDIF
  \ENDIF
  \UNTIL{\label{alg:exp:ln19}$\outc=\FALSE$}
\end{algorithmic}
%  \begin{comment}
%  \STATE{$\mbf{v}_{\tn{LB}}\gets\mbf{v}$}
%  \STATE{$\mbf{v}_{\tn{UB}}\gets\mbf{v}$}
%  %(v_1,\ldots,v_N)
%  \STATE{$(\fml{C},\fml{D},\fml{P})\gets(\fml{F},\emptyset,\emptyset)$}
%  \FORALL{$i\in\fml{F}$}
%  \STATE{$\mbf{v}_{\tn{LB}}\gets(v_{\tn{LB}_1},\ldots,\lambda(i),\ldots,v_{\tn{LB}_N})$}
%  \STATE{$\mbf{v}_{\tn{UB}}\gets(v_{\tn{UB}_1},\ldots,\mu(i),\ldots,v_{\tn{UB}_N})$}
%  \STATE{$(\fml{C},\fml{D})\gets(\fml{C}\setminus\{i\},\fml{D}\cup\{i\})$}
%  \IF{$\kappa(\mbf{v}_{\tn{LB}})\not=\kappa(\mbf{v}_{\tn{UB}})$}
%  \STATE{$(\fml{D},\fml{P})\gets(\fml{D}\setminus\{i\},\fml{P}\cup\{i\})$}
%  \STATE{$\mbf{v}_{\tn{LB}}\gets(v_{\tn{LB}_1},\ldots,v_i,\ldots,v_{\tn{LB}_N})$}
%  \STATE{$\mbf{v}_{\tn{UB}}\gets(v_{\tn{UB}_1},\ldots,v_i,\ldots,v_{\tn{UB}_N})$}
%  \ENDIF
%  \ENDFOR
%  \end{comment}

\end{algorithm}
The algorithm starts in \autoref{alg:exp:ln01} by initializing the CNF
formula $\fml{H}$ without clauses (these will be added as the 
algorithm executes).
The main loop (from \autoref{alg:exp:ln02} to \autoref{alg:exp:ln19})
is executed while the formula $\fml{H}$ is satisfiable. This is
decided with a call to a SAT reasoner (in~\autoref{alg:axp:ln03}).
Any satisfying assignment to the formula $\fml{H}$ partitions the
features into two sets: one denoting the features that can take any
value (with $u_i=1$) and another denoting the features that take the
corresponding value in $\mbf{v}$ (with $u_i=0$). (The assignment
effectively identifies a set $\fml{Z}\subseteq\fml{F}$, of fixed
features, and thus we can invoke~\autoref{prop:exp01}.)
In~\autoref{alg:exp:ln05} and~\autoref{alg:exp:ln06}, the algorithm
creates $\mbf{v}_{\tn{L}}$ and $\mbf{v}_{\tn{U}}$. For a fixed
feature $i$, both $\mbf{v}_{\tn{L}}$ and $\mbf{v}_{\tn{U}}$ are
assigned value $v_i$. For a free feature $i$, $\mbf{v}_{\tn{L}}$ and
$\mbf{v}_{\tn{U}}$ are respectively assigned to $\lambda(i)$ and
$\mu(i)$. Let $\fml{Z}$ denote the set of fixed features.
In~\autoref{alg:exp:ln07}, we check in which case of
\autoref{prop:exp01} applies.

If $\kappa(\mbf{v}_{\tn{L}})=\kappa(\mbf{v}_{\tn{U}})$, then we know
that the invariant of \autoref{alg:axp} holds. Moreover,
$\fml{F}\setminus\fml{Z}$ is a subset of an AXp. Hence, we set
$\fml{S}=\fml{F}\setminus\fml{Z}$ as the \emph{seed} for $\axp$.
This is shown in lines \ref{alg:exp:ln08} and \ref{alg:exp:ln09}.
After reporting the computed AXp, represented by the set of features
$\fml{P}$, we prevent the same AXp from being computed again by
requiring that at least one of the fixed features must be free in
future satisfying assignments of $\fml{H}$. This is represented as a
positive clause.

\begin{proposition}
  In the case $\kappa(\mbf{v}_{\tn{L}})=\kappa(\mbf{v}_{\tn{U}})$, set
  $\fml{S}$ is such that, for any previously computed AXp, at least
  one feature will be included in $\fml{S}$ (as a free literal). Since
  $\axp$ only grows $\fml{S}$, then the \autoref{alg:exp} does not
  repeat AXp's.
\end{proposition}

Moreover, if
$\kappa(\mbf{v}_{\tn{L}})\not=\kappa(\mbf{v}_{\tn{U}})$, then we know
that the invariant of \autoref{alg:cxp} holds. Moreover, $\fml{Z}$ is
a subset of a CXp. Hence, we set $\fml{S}=\fml{Z}$ as the \emph{seed}
for $\cxp$.
This is shown in lines \ref{alg:exp:ln13} and \ref{alg:exp:ln14}.
After reporting the computed CXp, represented by the set of features
$\fml{P}$, we prevent the same CXp from being computed by requiring
that at least one of the free features must be free in future
satisfying assignments of $\fml{H}$. This is represented as negative
clause.

\begin{proposition}
  In the case $\kappa(\mbf{v}_{\tn{L}})\not=\kappa(\mbf{v}_{\tn{U}})$, set
  $\fml{S}$ is such that, for any previously computed CXp, at least
  one feature will be included in $\fml{S}$ (as a fixed literal). Since
  $\cxp$ only grows $\fml{S}$, then the \autoref{alg:exp} does not
  repeat CXp's.
\end{proposition}

Given the %observations
above, and since the number of AXp's and CXp's
(being subsets of $\fml{F}$) is finite, then we have,

\begin{theorem}
  \autoref{alg:exp} is sound and complete for the enumeration of AXp's
  and CXp's.
\end{theorem}

%%%%%%%%%%%%%%%%%%%%%%%%%%%%%%%%%%%%%%%%%%%%%%%%%%%%%%%%%%%%%%%%%%%%%%%%%%%%%%%%
%% Moving table around... Used in  \autoref{ex01:expm}
\begin{table*}
  \begin{center}
    \scalebox{1.0}{ %0.95
      \renewcommand{\tabcolsep}{0.35em}
      \begin{tabular}{|c|c|c|cc|cc|cc|c|} \toprule
        Step & $\fml{H}$   &  $\mbf{u}~/~\msf{out}$ & 
        $\mbf{v}_{\tn{L}}$ & $\mbf{v}_{\tn{U}}$ &
        $\kappa(\mbf{v}_{\tn{L}})$ & $\kappa(\mbf{v}_{\tn{U}})$ &
        AXp & CXp & Clause added
        \\ \toprule
        1 & $\emptyset$ & $(0,0,1,0)$ &
        $(10,10,0,0)$ & $(10,10,10,0)$ & 
        $A$ & $A$ & $\{1,2\}$ & -- &
        $\begin{array}{l}(u_1\lor{u_2})\end{array}$
        %\\(\neg{u_3}\lor\neg{u_4})
        \\ \midrule
        2 & $\begin{array}{l}(u_1\lor{u_2})\end{array}$ & $(1,0,0,1)$ &
        %\\(\neg{u_3}\lor\neg{u_4})
        $(0,10,5,0)$ & $(10,10,5,10)$ & 
        $C$ & $A$ & -- & $\{1\}$ &
        $\begin{array}{l}(\neg{u_1})\end{array}$
        %\\(u_2\lor{u_3}\lor{u_4})
        \\ \midrule
        3 & $\begin{array}{l}(u_1\lor{u_2})\\(\neg{u_1})\end{array}$ & $(0,1,1,0)$ &
        %\\(\neg{u_3}\lor\neg{u_4})\\(u_2\lor{u_3}\lor{u_4})
        $(10,0,0,0)$ & $(10,10,10,0)$ & 
        $E$ & $A$ & -- & $\{2\}$ &
        $\begin{array}{l}(\neg{u_2})\end{array}$
        %\\(u_1\lor{u_3}\lor{u_4})
        \\ \midrule
        4 & $\begin{array}{l}(u_1\lor{u_2})\\(\neg{u_1}),(\neg{u_2})\end{array}$ & UNSAT &
        %\\(\neg{u_3}\lor\neg{u_4})\\(u_2\lor{u_3}\lor{u_4})\\(u_1\lor{u_3}\lor{u_4})
        -- & -- &  -- & -- & -- & -- & --
        \\ %\midrule
        \bottomrule
      \end{tabular}
    }
  \end{center}
  \caption{Execution of enumeration algorithm} \label{tab:exp-ex}
\end{table*}

%%%%%%%%%%%%%%%%%%%%%%%%%%%%%%%%%%%%%%%%%%%%%%%%%%%%%%%%%%%%%%%%%%%%%%%%%%%%%%%%

\begin{example} \label{ex01:expm02a}
  Building on earlier examples,
  %Examples \ref{ex01:def}, \ref{ex01:xps}, \ref{ex01:axpm}, and
  %\ref{ex01:cxpm}, this example
  we summarize the main steps of the SAT oracle-based algorithm for
  enumerating AXp and CXp explanations.
  \autoref{tab:exp-ex} illustrates one execution of the proposed
  algorithm.
  There are 1 AXp's and 2 CXp's. (Regarding the call to the SAT
  oracle, the satisfying assignments shown are intended to be as
  arbitrary as possible, given the existing constraints; other
  satisfying assignments could have been picked.)
  For each computed AXp, we add to $\fml{H}$ one positive clause. In
  this example, we add the clause $(u_1\lor{u_2})$, since the AXp is
  $\{1,2\}$. By adding this clause, we guarantee that features 1 and 2
  will not both be deemed fixed by subsequent satisfying assignments
  of $\fml{H}$.
  Similarly, for each computed CXp, we add to $\fml{H}$ one negative
  clause. For the example, the clauses added are $(\neg{u_1})$ for CXp
  $\{1\}$, and $(\neg{u_2})$ for CXp $\{2\}$. In both cases, the
  added clause guarantees that feature 1 (resp.~2) will not be deemed
  free by subsequent satisfying assignments of $\fml{H}$.
  One additional observation is that the number of SAT oracle calls
  matches the number of AXp's plus the number of CXp's and plus one
  final call to terminate the algorithm's execution.
  For step 4 of the algorithm, it is easy to conclude that $\fml{H}$
  is unsatisfiable, as intended.
\end{example}

\subsection{Related Work}

The algorithms for computing one AXp or one CXp for a monotonic
classifier are novel. However, the insight of analyzing elements of a
set (i.e.\ features in our case) to find a minimal set respecting some
property has been studied in a vast number of settings
(e.g.\ \cite{chinneck-bk08} for an overview). The proposed solution
for reasoning about ordinal features that can take boolean, integer or
real values, represents another aspect of novelty.
In the case of monotonic classifiers, we obtain a running time that is
linear in the running time complexity of the classifier. This result
applies in the case of \emph{any} monotonic classifier, and so it
improves significantly over the worst-case exponential time and space
approach proposed in earlier work~\cite{darwiche-ijcai18}, for the
concrete case of monotonic bayesian networks.
The algorithm for enumerating AXp's and CXp's for a monotonic
classifier is also novel. However, it is inspired by the MARCO
algorithm for the analysis of inconsistent logic
theories~\cite{lpmms-cj16}. Although MARCO can be optimized in
different ways, \autoref{alg:exp} can be related with its most basic
formulation. Since computing one AXp or one CXp can be
achieved in polynomial time (conditioned by the classifier run time
complexity), then our approach guarantees that exactly one SAT
reasoner call is required for each computed minimal set (i.e.\ AXp or
CXp in our case).

\newcommand{\monoxp}{\texttt{MonoXP}\xspace}

\section{Experiments} \label{sec:res}

The objective of this section is to illustrate the scalability of both
the algorithms for finding one explanation, but also the algorithm for
enumerating explanations. The tool \monoxp implements the
algorithms~\ref{alg:axp},~\ref{alg:cxp} and~\ref{alg:exp}.
As observed in recent work, most monotonic classifiers are
not publicly available~\cite{garcia-neuroc19}. We analyze two
publicly available classifiers, and describe two experiments. The first
experiment evaluates \monoxp for explaining two recently proposed
tools, COMET%
\footnote{Available from\\\url{https://github.com/AishwaryaSivaraman/COMET}.}%
~\cite{vandenbroeck-nips20}
and monoboost%
\footnote{Available from\\\url{https://github.com/chriswbartley/monoboost}.}%
~\cite{reynolds-icde18}. %
%%\footnote{COMET is available from~\cite{comet-tool}.}.
COMET is run on the Auto-MPG%
\footnote{\url{https://archive.ics.uci.edu/ml/datasets/auto+mpg}.}
dataset studied in earlier work~\cite{vandenbroeck-nips20}, with the
choice justified by the time the classifier takes to run.
monoboost is run on a monotonic dataset with two classes (as required
by the tool)~\cite{reynolds-icde18}. We use a monotonic subset
(PimaMono) of the Pima
dataset\footnote{\url{https://sci2s.ugr.es/keel/dataset.php?cod=21}.}.
A second experiment compares \monoxp with
Anchor~\cite{guestrin-aaai18}, both in terms of the number of calls to
the classifier and running time, but also in terms of the quality of
the computed explanations\footnote{%
  It should be underlined that neither Anchor~\cite{guestrin-aaai18},
  LIME~\cite{guestrin-kdd16} nor SHAP~\cite{lundberg-nips17} can
  enumerate explanations, neither can these tools compute heuristic
  contrastive explanations.}, namely accuracy and size.
This second experiment also considers two datasets.
The first dataset is BankruptcyRisk~\cite{greco-bk98} (which is
monotonic if one instance is dropped). For this dataset, the monotonic
decision tree classifier proposed in earlier work is
used~\cite{bioch-idaj00}.
The second dataset is PimaMono, and the classifier used is the one
obtained with monoboost (as in the first experiment).
All experiments were run on a MacBook Pro, with a 2.4GHz quad-core i5
processor, and 16 GByte of RAM, running MacOS Big Sur.
For each dataset, we either pick 100 instances, randomly selected, or
the total number of instances in the dataset, in case this number does
not exceed 100.

%\begin{figure*}[t]
\begin{table*}[t]
  \begin{subtable}{\textwidth}
    %\begin{subfigure}{\textwidth}
    %%\begin{table*}[t]
    \begin{center}
      \renewcommand{\tabcolsep}{0.35em}
\begin{tabular}{|c|c|cccccc|} \toprule
    Dataset/Tool & \#Inst. &
    Avg.~\#~expl. &
    Avg.~AXp~sz &
    Avg.~CXp~sz &
    Avg.~classif.~time &
    Avg.~run time &
    \% classif. time
    \\ \toprule
    AutoMPG/CMT &
    100 &
    2.35 &
    1.49 &
    1.02 &
    105.90s &
    105.92s &
    99.99\%
    \\ \midrule
    PimaMono/MBT &
    69 &
    9.09 &
    1.27 &
    3.36 &
    16.285s &
    16.360s &
    99.54\%
    \\ \bottomrule
\end{tabular}

    \end{center}
    \caption{Assessing \monoxp on the Auto-MPG and PimaMono datasets,
      using resp.~COMET or monoboost as the classifier} \label{tab:res01}
    %%\end{table*}
    %\end{subfigure}
  \end{subtable}

  \smallskip\smallskip

  \begin{subtable}{\textwidth}
    %\begin{subfigure}{\textwidth}
    %%\begin{table*}[t]
    \begin{center}
      \renewcommand{\tabcolsep}{0.375em}
\begin{tabular}{|c|c|ccc|cccc|c|}
  \toprule
  \multirow{2}{*}{Dataset} &
  \multirow{2}{*}{\#Inst.} &
  \multicolumn{3}{c|}{Anchor} &
  \multicolumn{4}{c|}{\monoxp (AXp)} &
  %\multicolumn{1}{c|}{Comparison}
  \multirow{2}{*}{\% diff}
  \\ \cline{3-9}
  & &
  Avg.~Xp~sz & Avg.~time & \# Cls calls &
  Avg.~\#~Xp & Avg.~Xp~sz & Avg.~Xp~time & \# Cls calls &
  %%\% diff %%& \# diff.~feat. %% & ($>$,$=$,$<$) 
  \\ \toprule
  B.~Risk & 39 &
  2.18 & 0.11s & 1217 &
  1.03 & 2.0 & 0.009s & 24 &
  64.1 %%& (7,32,0)
  \\ \midrule
  PimaMono & 69 &
  1.26 & 11.2s & 2967 &
  5.64 & 1.27 & 1.8s & 16 &
  18.8 %%& (~~,~~,~~)
  \\ \bottomrule
\end{tabular}

    \end{center}
    \caption{Assessing \monoxp and Anchor on the Bankruptcy Risk
      and Pima Mono datasets} \label{tab:res02}
    %%\end{table*}
    %\end{subfigure}
  \end{subtable}
  \caption{Results of running \monoxp} \label{fig:restabs}
\end{table*}
%\end{figure*}

\subsection{Cost of Computing Explanations}

We run \monoxp on a neural network classifier envelope implemented with
COMET for the Auto-MPG dataset, and on a tree ensemble obtained with
monoboost for the PimaMono dataset.
% Given the amount of time COMET takes per instance, of the two
%datasets considered in recent work~\cite{vandenbroeck-nips20}, we
%select the Auto-MPG dataset~\cite{auto-mpg},
%\footnote{\url{https://archive.ics.uci.edu/ml/datasets/auto+mpg}.},
%since the classification time is significantly smaller.
%
(Since the running times of COMET can be significant, this
experiment does not consider a comparison with the heuristic explainer
Anchor~\cite{guestrin-aaai18}. As shown below, Anchor calls
the classifier a large number of times, and that would imply unwieldy
running times.)

\autoref{tab:res01} shows the results of running \monoxp using COMET
as a monotonic envelope on the Auto-MPG dataset, and monoboost on the
PimaMono dataset.
As can be observed, the explanation sizes are in general small, which 
confirms the interpretability of computed AXp's and CXp's.
As a general trend, CXp's are on average smaller than AXp's for
Auto-MPG, but larger than AXp's for PimaMono.
Moreover, the classification time completely dominates the total
running time (i.e.\ resp.~99.99\% and 99.54\% of the time is spent
running the classifier, independently of the classifier used).
These results offer evidence that the time spent on computing
explanations is in general negligible for monotonic classifiers.
For both datasets, and for the instances considered, it was possible
to enumerate all AXp and CXp explanations, with negligible
computational overhead.

\subsection{Comparison with Anchor}

%The second experiment
This section compares \monoxp with Anchor, using 
two pairs of classifiers and datasets, i.e.\ a monotonic decision tree
for BankruptcyRisk and monoboost for PimaMono.

%Given the very large time spent on classification with COMET, we
%opted instead to pick a smaller simpler monotonic classifier, for
%which a hard-coded classifier could be used. (This way, the time spent
%on classification is negligible, and so we are comparing solely the
%explanation times.)
%
%As a result, we opted for a monotonic decision
%classifier~\cite{bioch-idaj00}, obtained on a mostly monotonic
%Bankruptcy Risk (BR) dataset~\cite{greco-bk98} (the dataset becomes 
%monotonic by removing one instance~\cite{bioch-idaj00}).

\autoref{tab:res02} shows the results of running Anchor and \monoxp on
the BankruptcyRisk and the PimaMono datasets.
\monoxp is significantly faster than Anchor (more than 1 order
magnitude in the first case, and more than a factor of 5 in the second
case). 
The justification is the much smaller number of calls to the
classifier required by \monoxp than by Anchor. (While for Anchor the
number of calls to the classifier can be significant, for \monoxp,
each AXp is computed with at most a linear number of calls to the
classifier. Thus, unless the number of features is very substantial,
\monoxp has a clear performance edge over Anchor.)
Somewhat surprisingly, over all instances, the average size of AXp's
computed by \monoxp is smaller than that of Anchor for the
BankruptcyRisk dataset. For the PimaMono dataset, the average size is
almost the same. These results suggest that formally defined
explanations need not be significantly larger than the ones computed
with heuristic approaches.
%This is confirmed with 7 instances for which the explanation of Anchor
%is larger than that of \monoxp, while no case exists where the size of
%AXp's computed by \monoxp exceeds that of Anchor.
%
Furthermore, for 64.1\% (resp.\ 18.8\%) of the instances, Anchor
identifies an explanation that does not hold across all points of
feature space, i.e.\ there are points in feature space for which the
explanation of Anchor holds, but the ML model makes a different
prediction\footnote{%
  Similar observations have been reported
  elsewhere~\cite{ignatiev-ijcai20}.}.
Observe that since \monoxp computes all AXp's, we can be certain about
whether the explanation of Anchor is a correct explanation.

\section{Conclusions \& Discussion} \label{sec:conc}

This paper proposes novel algorithms for computing a single PI or
contrastive explanation for a monotonic classifier. In contrast with
earlier work~\cite{darwiche-ijcai18}, the complexity of the proposed
algorithms is polynomial on the number of features and the time it
takes the monotonic classifier to compute its predictions. As the
experiments demonstrate, for simple ML models, the algorithm achieves
one order of magnitude speed up when compared with a well-known
heuristic explainer~\cite{guestrin-aaai18}, achieving better quality
explanations of similar size. In contrast, for complex ML models, the
experiments confirm that the running time is almost entirely spent on
the classifier.
Furthermore, the paper proposes a SAT-based algorithm for enumerating
PI and contrastive explanations. As the experimental results show, the
use of a SAT solver for enumerating PI and contrastive explanations
incurs a negligible overhead.

One possible criticism of the work is that SAT solvers are used for
guiding the enumeration of explanations. This involves solving
an NP-complete decision problem for each computed explanation, and so
it might pose a scalability concern. (One alternative would be to
consider explicit enumeration of candidate explanations, as proposed
in the earlier works on model based
diagnosis~\cite{reiter-aij87,greiner-aij89,wotawa-ipl01}.)
However, for classification problems with tens to hundreds of
features and targeting thousands to tens of thousands explanations
(and this far exceeds currently foreseen scenarios), the use of modern
SAT reasoners (capable of solving problems with hundreds of thousands
of variable and millions of clauses) can hardly be considered a
limitation. 
Another possible criticism of this work is that full monotonicity is 
required. We conjecture that \emph{full} monotonicity is necessary for
tractable explanations (conditioned by the classifier run time
complexity). Addressing partial monotonicity ~\cite{velikova-tnn10} is
a subject of future research.

\paragraph{Acknowledgments.} This work was supported by the AI
Interdisciplinary Institute ANITI, funded by the French program
``Investing for the Future -- PIA3'' under Grant agreement
no.\ ANR-19-PI3A-0004, and by the H2020-ICT38 project COALA
``Cognitive Assisted agile manufacturing for a Labor force supported
by trustworthy Artificial intelligence''.

%\clearpage
% References can use any number of pages
%\bibliography{refs,apps}
\bibliographystyle{icml2021}

%\clearpage
\appendix

\section{Appendix} %Supplementary Material

\begin{reptheorem}{thm:exp01}
  Determining the existence of $\lfloor N/2 \rfloor{+}1$ AXp's
  (or CXp's) of a monotonic $N$-feature classifier is NP-complete.
\end{reptheorem}

%\begin{theorem} \label{thm:exp01}
%  Determining the existence of $\lfloor N/2 \rfloor{+}1$ AXp's
%  (or CXp's) of a monotonic $N$-feature classifier is NP-complete.
%\end{theorem}

\begin{proof}
We say that a CNF is trivially satisfiable if some literal
occurs in all clauses. Clearly, SAT restricted to non-trivial
CNFs is still NP-complete. Let $\Phi$ be a not trivially-satisfiable
CNF on variables $x_1,\ldots,x_k$. Let $N=2k$. Let $\tilde{\Phi}$
be identical to $\Phi$ except that each occurrence of a negative literal
$\overline{x_i}$ ($1 \leq i \leq k$) is replaced by $x_{i+k}$.
Thus $\tilde{\Phi}$ is a CNF on $N$ variables each of which occur
only positively. Define the boolean classifier $\kappa$ by
$\kappa(x_1,\ldots,x_N)=1$ if $x_i=x_{i+k}=1$ 
for some $i \in \{1,\ldots,k\}$ or $\tilde{\Phi}(x_1,\ldots,x_N)=1$
(and 0 otherwise). To show that $\kappa$ is monotonic we
need to show that 
$\mbf{a}\le\mbf{b}\Rightarrow\kappa(\mbf{a})\leq\kappa(\mbf{b})$.
This follows by examining the two cases in which $\kappa(\mbf{a})=1$:
if $a_i{=}a_{i+k} \land \mbf{a}{\le}\mbf{b}$, then 
$b_i{=}b_{i+k}$, whereas,
if $\tilde{\Phi}(\mbf{a}){=}1 \land \mbf{a}{\le}\mbf{b}$, then
$\tilde{\Phi}(\mbf{b})=1$ (by positivity of $\tilde{\Phi}$),
so in both cases $\kappa(\mbf{b})=1 {\geq} \kappa(\mbf{a})$.

We first consider AXp's.
Clearly $\kappa(\mbf{1})=1$. There are $N/2$ obvious AXp's
of this prediction, namely $(i,i{+}k)$ ($1{\leq}i{\leq}k$). These are
minimal by the assumption that $\Phi$ is not trivially satisfiable.
Suppose that $\Phi(\mbf{u}){=}1$. Let $\fml{X}_{\mbf{u}}$ be 
$\{i \, | \, 1 {\leq} i {\leq} k \land u_i{=}1\} 
\cup \{i{+}k \, | \,  1 {\leq} i {\leq} k \land u_{i}{=}0\}$. 
Then (some subset of) $\fml{X}_{\mbf{u}}$ is an AXp of the prediction
$\kappa(\mbf{1}){=}1$. The converse also holds. 
Thus, determining whether $\kappa(\mbf{1}){=}1$
has more than $N/2$ AXp's is equivalent to testing the 
satisfiability of $\Phi$. NP-completeness follows from
the fact that $\lfloor N/2 \rfloor{+}1$ AXp's are a
polytime verifiable certificate.

The proof for CXp's is similar. Clearly $\kappa(\mbf{0})=0$.
Again, there are $N/2$ obvious CXp's
of this prediction, namely $(i,i{+}k)$ ($1{\leq}i{\leq}k$)
and (some subset of) $\fml{X}_{\mbf{u}}$ is a CXp 
iff $\Phi(\mbf{u}){=}1$. Thus, determining whether $\kappa(\mbf{0}){=}0$
has more than $N/2$ CXp's is equivalent to testing the 
satisfiability of $\Phi$, from which NP-completeness again follows.
\end{proof}

In the case of AXp's, the theorem follows from a result
on boolean monotone functions~\cite{DBLP:conf/icfca/BabinK11}, 
but for clarity of exposition we preferred to give a direct proof.

\end{document}